\documentclass[11pt]{amsart}
\usepackage[margin=1in]{geometry}
\usepackage{hyperref}
\usepackage{xcolor}
\hypersetup{
    colorlinks,
    linkcolor={red!50!black},
    citecolor={blue!50!black},
    urlcolor={blue!80!black}
}

\usepackage{amssymb}
\usepackage{mathtools}
\usepackage{amsthm}
\usepackage{enumitem}

\theoremstyle{plain}
\newtheorem{theorem}{Theorem}[section]

\newtheorem{lemma}[theorem]{Lemma}
\newtheorem{corollary}[theorem]{Corollary}
\newtheorem{conjecture}[theorem]{Conjecture}
\theoremstyle{definition}
\newtheorem{definition}[theorem]{Definition}

\theoremstyle{remark}

\usepackage[textsize=tiny]{todonotes}
\usepackage{caption}
\usepackage{subcaption}
\usepackage{tikz}
\usepackage{xcolor}
\usepackage{listofitems}
\usepackage{pgfplots}
\usetikzlibrary{arrows.meta}
\usepackage[outline]{contour}
\usetikzlibrary{positioning}

\colorlet{myred}{red!80!black}
\colorlet{myblue}{blue!80!black}
\colorlet{mygreen}{green!80!black}
\colorlet{myorange}{orange!70!red!60!black}
\colorlet{mydarkred}{red!20!black}
\colorlet{mydarkblue}{blue!20!black}
\colorlet{mydarkgreen}{green!20!black}

\tikzset{
  >=latex, 
  node/.style={thick,circle,draw=myblue,minimum size=22,inner sep=0.5,outer sep=0.6},
  node in/.style={node,green!20!black,draw=mygreen!30!black,fill=mygreen!25},
  node hidden/.style={node,blue!20!black,draw=myblue!30!black,fill=myblue!20},
  node convol/.style={node,orange!20!black,draw=myorange!30!black,fill=myorange!20},
  node out/.style={node,red!20!black,draw=myred!30!black,fill=myred!20},
  connect/.style={thick,mydarkblue}, 
  connect arrow/.style={-{Latex[length=4,width=3.5]},thick,mydarkblue,shorten <=0.5,shorten >=1},
  node 1/.style={node in}, 
  node 2/.style={node hidden},
  node 3/.style={node out}
}
\usetikzlibrary{positioning,shapes,arrows,chains}
\tikzset{arrow/.style={-stealth, thick, draw=gray!80!black}}
\definecolor{lcfree}{RGB}{252,224,225}
\definecolor{emb}{RGB}{252,200,200}
\definecolor{att}{RGB}{200,252,200}
\definecolor{LN}{RGB}{200,252,150}
\definecolor{FFN}{RGB}{203,252,250}
\definecolor{SoftMax}{RGB}{203,203,250}
\definecolor{RC}{RGB}{250,200,152}
\definecolor{Linear}{RGB}{203,252,203}
\definecolor{lcnorm}{RGB}{152,152,152}

\newcommand{\tp}{{\scriptscriptstyle\mathsf{T}}}
\newcommand{\ih}{\hat{\imath}}
\newcommand{\jh}{\hat{\jmath}}
\newcommand{\ib}{\bar{\imath}}
\newcommand{\jb}{\bar{\jmath}}

\DeclareMathOperator{\ReLU}{ReLU}
\DeclareMathOperator{\SoftMax}{SoftMax}
\DeclareMathOperator{\SoftPlus}{SoftPlus}
\DeclareMathOperator{\argmax}{argmax}
\DeclareMathOperator{\mask}{mask}
\DeclareMathOperator{\sgn}{sgn}

\let\latexcirc=\circ
\newcommand{\ccirc}{\mathbin{\mathchoice
  {\xcirc\scriptstyle}
  {\xcirc\scriptstyle}
  {\xcirc\scriptscriptstyle}
  {\xcirc\scriptscriptstyle}
}}
\newcommand{\xcirc}[1]{\vcenter{\hbox{$#1\latexcirc$}}}
\let\circ\ccirc

\begin{document}
\title{Attention is a smoothed cubic spline}
\author[Z.~Lai]{Zehua~Lai}
\address{Department of Mathematics, University of Texas, Austin, TX 78712}
\email{zehua.lai@austin.utexas.edu}
\author[L.-H.~Lim]{Lek-Heng~Lim}
\address{Computational and Applied Mathematics Initiative, Department of Statistics,
University of Chicago, Chicago, IL 60637}
\email{lekheng@uchicago.edu}
\author[Y.~Liu]{Yucong~Liu}
\address{School of Mathematics, Georgia Institute of Technology, Atlanta, GA 30332}
\email{yucongliu@gatech.edu}

\begin{abstract}
We highlight a perhaps important but hitherto unobserved insight: The attention module in a transformer is a smoothed cubic spline. Viewed in this manner, this mysterious but critical component of a transformer becomes a natural development of an old notion deeply entrenched in classical approximation theory. More precisely, we show that with ReLU-activation, attention, masked attention, encoder--decoder attention are all cubic splines.  As every component in a transformer is constructed out of compositions of various attention modules (= cubic splines) and feed forward neural networks (= linear splines), all its components --- encoder, decoder, and encoder--decoder blocks; multilayered encoders and decoders; the transformer itself --- are cubic or higher-order splines. If we assume the Pierce--Birkhoff conjecture, then the converse also holds, i.e., every spline is a ReLU-activated encoder. Since a spline is generally just $C^2$, one way to obtain a smoothed $C^\infty$-version is by replacing ReLU with a smooth activation; and if this activation is chosen to be SoftMax, we recover the original transformer as proposed by Vaswani et al. This insight sheds light on the nature of the transformer by casting it entirely in terms of splines, one of the best known and thoroughly understood objects in applied mathematics.
\end{abstract}
\maketitle

The transformer \cite{vaswani2017attention} underlies many modern  AI technologies in the current news cycle. Splines, on the other hand, are among the oldest tools in classical approximation theory, studied since the 1940s, and culminated in the 1980s \cite{de2009way} before taking on a new life in the form of wavelets (e.g., the celebrated Cohen--Daubechies--Feauveau wavelet \cite{daub} that underlies JPEG 2000 compression comes from a B-spline). Indeed, the word ``spline'' originally refers to the flexible wooden strip that serves as a bendable ruler for shipbuilders and draftsmen to draw smooth shapes since time immemorial; the Wright brothers had notably used such wooden splines to design their aircraft. It is therefore somewhat surprising that a notion so old is nearly one and the same as a notion so new --- we will show that every $\ReLU$-activated attention module $F : \mathbb{R}^{n \times p} \to \mathbb{R}^{m \times p}$ is a multivariate cubic spline, and, if we assume a conjecture of Garrett Birkhoff and Richard Pierce from 1956 \cite{BP}, then  conversely every multivariate spline $G : \mathbb{R}^m \to \mathbb{R}^n$ is a $\ReLU$-activated encoder. The usual $\SoftMax$-activated attention module is thus a simple and natural way to make a cubic spline, which is at most a $C^2$-function, into a smooth function --- by replacing the nonsmooth $\ReLU$ with a smooth $\SoftMax$.

Why did approximation theorists not discover the transformer then? We posit that it is due to a simple but fundamental difference in how they treat the decomposition of a complicated function into simpler ones. In approximation theory and harmonic analysis, one decomposes a complicated function $F$ into a  \emph{sum} of simpler functions $f_1,\dots,f_r$,
\begin{equation}\label{eq:sum}
F = f_1 + f_2 + \dots + f_r;
\end{equation}
in artificial intelligence, one decomposes $F$ into a \emph{composition} of simpler functions $F_1,\dots,F_r$,
\begin{equation}\label{eq:chain}
F = F_1 \circ F_2 \circ \dots \circ F_r. 
\end{equation}
This fundamental difference in modeling a function is a key to the success of modern AI models. Suppose $F : \mathbb{R}^n \to \mathbb{R}^n$. If we model $F$ as a sum in \eqref{eq:sum}, the number of parameters scales like $nd^n$:
\[
F(x) = \sum_{i_1=1}^d \dots \sum_{i_n=1}^d \sum_{j=1}^n a_{i_1 i_2 \cdots i_nj} \varphi_{i_1}(x_1) \varphi_{i_2}(x_2) \cdots \varphi_{i_n}(x_n) e_j;
\]
whereas if we model $F$ as a composition in \eqref{eq:chain}, it scales like $dn^2 + (d-1)n$:
\[
F(x) = A_d \sigma_{d-1}  A_{d -1}  \cdots  \sigma_2 A_2  \sigma_1  A_1x
\]
with $A_i \in \mathbb{R}^{n \times n}$, $\sigma_i$  parameterized by a vector in  $\mathbb{R}^n$. Note that even if $d=2$, the size $n2^n$ quickly becomes untenable. Evidently, these ball park estimates are made with some assumptions:
The root cause of this notorious \emph{curse of dimensionality} is that there are no good general ways to construct the basis function $f_i$ in \eqref{eq:sum} except as a tensor product of low (usually one) dimensional basis functions, i.e., as $\varphi_{i_1} \otimes \varphi_{i_2} \otimes \dots \otimes \varphi_{i_n} \otimes e_j$. The compositional model \eqref{eq:chain} allows one to circumvent this problem beautifully. Take the simplest case of a $d$-layer feed forward neural network, as we did above; then it is well known that $d$ can be small \cite{cybenko1989approximation,kidger2020universal}.

An important feature of \eqref{eq:sum} and \eqref{eq:chain} is that both work well with respect to derivative by virtue of linearity
\[
DF = Df_1 + Df_2 + \dots + Df_r
\]
or chain rule
\[
DF = DF_1 \circ DF_2 \circ \dots \circ DF_r.
\]
The former underlies techniques for solving various PDEs, whether analytically or numerically; the latter underlies the back-propagation algorithm for training various AI models (where the former also plays a role through various variants of the stochastic gradient descent algorithm). The bottom line is that the conventional way to view a cubic spline, as a sum of polynomials supported on disjoint polygonal regions or a sum of monomials, takes the form in \eqref{eq:sum}. A $\ReLU$-attention module is just the same cubic spline expressed in the form \eqref{eq:chain}, and in this form there is a natural and straightforward way to turn it into a smooth function, namely, replace all nonsmooth $F_i$'s with smooth substitutes --- if we replace $\ReLU$ by $\SoftMax$, we obtain the attention module as defined in \cite{vaswani2017attention}. This is a key insight of our article. 

It is well-known \cite{arora2018understanding} that a $\ReLU$-activated feed forward neural network may be viewed as a \emph{linear spline} expressed in the form of \eqref{eq:chain}. When combined with our insight that a $\ReLU$-activated attention module is a \emph{cubic spline}, we deduce that every other intermediate components of the $\ReLU$-transformer --- encoder, decoder, encoder--decoder --- are either cubic or higher-order spline, as they are constructed out of compositions and self-compositions of $\ReLU$-activated feed forward neural networks  and $\ReLU$-activated attention modules.

A word of caution: We are not claiming that $\SoftMax$ would be a natural smooth replacement for $\ReLU$. We will touch on this in Section~\ref{sec:conc}. Indeed, according to recent work \cite{wortsman2023replacing}, this replacement may be wholly unnecessary --- when it comes to transformers, $\ReLU$ would be an equally if not superior choice of activation compared with $\SoftMax$.

\subsection{Understanding transformers via splines}

Our main contribution is to explain a little-understood new technology using a well-understood old one. For the benefit of approximation theorists who may not be familiar with transformers or machine learning theorists who may not be familiar with splines, we will briefly elaborate.

The transformer has become the most impactful technology driving AI. It has revolutionized natural language processing, what it was originally designed for \cite{vaswani2017attention}, but by this point there is no other area in AI, be it computer vision \cite{dosovitskiy2021an}, robotics \cite{pmlr-v229-zitkovich23a}, autonomous vehicles \cite{prakash2021multi}, etc, that is left untouched by transformers. This phenomenal success is however empirical, the fundamental principles underlying the operation of transformers have remained elusive.

The attention module is evidently the most critical component within a transformer, a fact reflected in the title of the paper that launched the transformer revolution \cite{vaswani2017attention}. It is arguably the only new component --- the remaining constituents of a transformer are $\ReLU$-activated feed forward neural networks, which have been around for more than 60 years \cite{Rosenblatt} and thoroughly investigated. Unsurprisingly, it is also the least understood. An attention module is still widely understood by way of ``query, key, value'' and a transformer as a flow chart, as in the article where the notion first appeared \cite{vaswani2017attention}. The main goal of our article is to understand the attention module in particular and the transformer in general, by tying them to one of the oldest and best-understood object in approximation theory.

Splines are a mature, well-understood technology that has been thoroughly studied and widely used \cite{Schoenberg1946, SchoenbergPNAS, dw1976Splines, de1978practical, BoxSpline, Chui, Wahba, Shikin, Handbook, de2009way}, one of our most effective and efficient methods for approximating known functions and interpolating unknown ones. They have numerous applications and we will mention just one: representing intricate shapes in computer graphics and computer-aided design. Readers reading a hard copy of this article are looking at fonts whose outlines are defined by splines \cite{metafont}; those viewing it on screen are in addition using a device likely designed with splines \cite{HCAD}. Splines have ushered in a golden age of approximation theory, and were studied extensively c.\ 1960--1980, until wavelets supplanted them. One could not have asked for a better platform to understand a new technology like the attention module and transformer.

Nowhere is this clearer than our constructions in Section~\ref{sec:s=t} to show that every spline is an encoder of a $\ReLU$-transformer. These constructions reveal how each feature of the transformer --- attention,  heads, layers, feed forward neural networks --- plays an essential role. We made every attempt to simplify and failed: Omit any of these features and we would not be able to recreate an arbitrary spline as an encoder. It were as if the inventors of transformer had designed these features not with any AI applications in mind but to construct splines as compositions of functions.

\section{Mathematical description of the transformer}\label{sec:math}

A transformer is typically presented in the literature as a flow chart \cite[Figure~1]{vaswani2017attention}.  We show a version in Figure~\ref{fig:trans}.

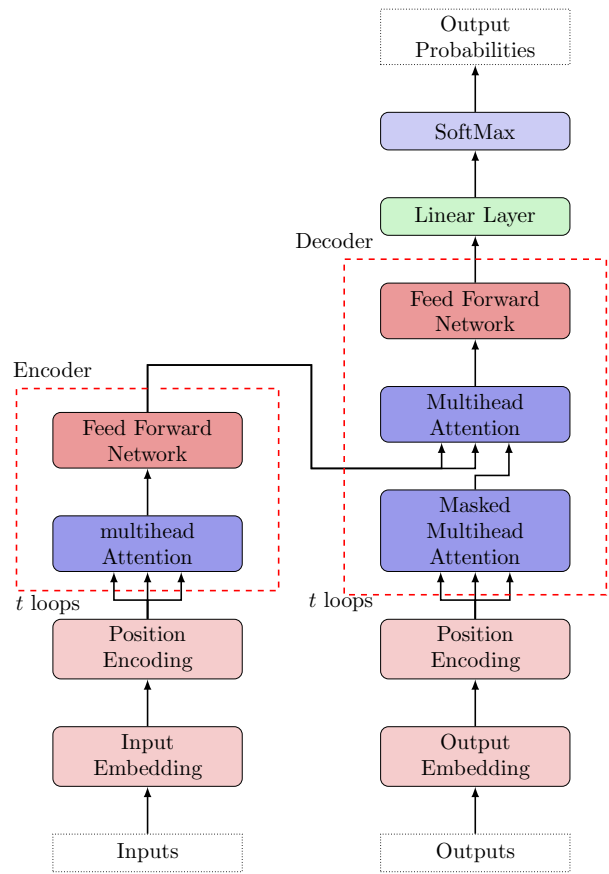
\begin{figure}[htb]
    \centering
    \resizebox{0.49\textwidth}{!}{
    \begin{tikzpicture}[
        start chain=going above,
        node distance=5ex and 15em,
        every join/.style={->,thick},
        ]
    \tikzset{
      base/.style={draw, on chain, on grid, align=center, minimum height=4ex},
      proc/.style={base, rectangle, text width=8em},
      test/.style={base, diamond, aspect=2, text width=5em},
      term/.style={proc, rounded corners},
      coord/.style={coordinate, on chain, on grid, node distance=6ex and 15em},
      nmark/.style={draw, cyan, circle, font={\sffamily\bfseries}},
      norm/.style={->, draw, lcnorm},
      it/.style={font={\small\itshape}}
    }
    \node[proc, densely dotted](p1){Inputs};
    \node[term, fill= myred!20, join]{Input \\Embedding};
    \node[term, fill = myred!20,join](p6){Position \\Encoding};
    \node[term, fill = myblue!40,join](p2){multihead\\Attention};
    \node[term, fill = myred!40,join](p3){Feed Forward\\Network};
   
    \node[proc, densely dotted, right= of p1](t1){Outputs};
    \node[term, fill= myred!20, join]{Output \\Embedding};
    \node[term, fill = myred!20,join](t8){Position \\Encoding};
    \node[term, fill = myblue!40,join](t2){Masked Multihead\\Attention};
    \node[term, fill = myblue!40](t3){Multihead\\Attention};
    \node[term, fill = myred!40,join](t4){Feed Forward\\Network};
    \node[term, fill = mygreen!20,join]{Linear Layer};
    \node[term, fill = myblue!20,join]{SoftMax};
    \node[proc, densely dotted,join]{Output \\Probabilities};
   
    \draw[->,thick](p3.north) -- ++(0em,5ex) -- ++(7.5em,0ex) -- ++(0em,-11ex) -|(t3.south);
   
    \draw[->,thick](p3.north) -- ++(0em,5ex) -- ++(7.5em,0ex) -- ++(0em,-11ex) -|(t3.220);
    \path (p3.west) to node [near start,  xshift=0em,yshift=7.5ex] {Encoder} (p3);
    \path (p2.west) to node [near start,  xshift=-1em,yshift=-8.5ex] {$t$ loops} (p3);
    \draw[red,thick,dashed] (p3.north) ++(0em,2.5ex) -- ++(-6em,0ex) -- ++(0em,-21.5ex) -- ++(12em,0ex) -- ++(0em,+21.5ex) -- ++(-5.5em,0ex);
    
    \draw[->,thick] (p6.north) -- ++(0em,2ex) -|(p2.220);
    
    \draw[->,thick] (p6.north) -- ++(0em,2ex) -|(p2.320);
    
    \draw[->,thick] (t8.north) -- ++(0em,2ex) -|(t2.230);
    
    \draw[->,thick] (t8.north) -- ++(0em,2ex) -|(t2.310);
    
    \draw[->,thick] (t2.north) -- ++(0.0em,1.5ex) -|(t3.320);
    
    \path (t4.east) to node [near start,  xshift=-10.5em,yshift=-10.8em] {$t$ loops} (t2.east);
    \path (t4.east) to node [near start,  xshift=-10em,yshift=9.5ex] {Decoder} (t3);
    \draw[red,thick,dashed] (t4.north) ++(0em,2.5ex) -- ++(6em,0ex) -- ++(0em,-35.7ex) -- ++(-12em,0ex) -- ++(0em,35.7ex) -- ++(6em,0ex);
    \end{tikzpicture}}
    \caption{Transformer as flow chart.}
    \label{fig:trans}
\end{figure}

Without a rigorous definition of the transformer, it will be difficult if not impossible to prove mathematical claims about it.  We will nail down in mathematically precise terms the full inner workings of a transformer. While it is common to find descriptions that selectively present parts as well-defined maps and revert to words and pictures when it becomes less convenient, what sets us apart below is thoroughness --- nothing will be swept under the rug. On occasions we had to look into the source codes of common implementations to unravel inconvenient details left ambiguous in the literature.  This section is our small side contribution and a public service. 

The heart of Figure~\ref{fig:trans} are the two parts enclosed in red dash lines, called encoder and decoder respectively. They are constructed out of feed forward neural networks, defined in Section~\ref{sec:nn}, and attention modules, defined in Section~\ref{sec:att}, chained together via function compositions. The simplest version is the encoder in Section~\ref{sec:enc} and is what the uninitiated reader should keep in mind. We add the bells and whistles later: Section~\ref{sec:mask} defines the \emph{masked} attention in the right-half of Figure~\ref{fig:trans}, from which we obtain the decoder in Section~\ref{sec:dec}. Section~\ref{sec:encdec} explains the encoder--decoder structure --- the left- and right-halves in Figure~\ref{fig:trans}. Section~\ref{sec:trans} puts everything together to define the transformer.  Section~\ref{sec:norm} discusses the one omission in Figure~\ref{fig:trans}, the ``add \& norm'' layers found in  \cite[Figure~1]{vaswani2017attention}.

\subsection{Notations}\label{sec:notate}

We write all vectors in $\mathbb{R}^n$ as column vectors, i.e., $\mathbb{R}^n \equiv \mathbb{R}^{n \times 1}$. Let $x_1,\dots,x_n \in \mathbb{R}$. When enclosed in parentheses $(x_1,\dots,x_n)$ denotes a \emph{column} vector, i.e.,
\[
(x_1,\dots,x_n) \coloneqq \begin{bmatrix} x_1 \\ \vdots \\ x_n \end{bmatrix} \in \mathbb{R}^n.
\]
When enclosed in brackets $[x_1,\dots,x_n] \in \mathbb{R}^{1 \times n}$ is a \emph{row} vector.

We will apply this convention more generally: For matrices $X_1,\dots,X_n \in \mathbb{R}^{m \times p}$, we write
\[
(X_1,\dots,X_n) \coloneqq \begin{bmatrix} X_1 \\ \vdots \\ X_n \end{bmatrix} \in \mathbb{R}^{mn \times p}
\]
and $[X_1,\dots,X_n] \in \mathbb{R}^{m \times np}$.

When we write $(f_1,\dots,f_h)$ for functions $f_i : \mathbb{R}^{n \times p} \to \mathbb{R}^{m \times p}$, $i=1,\dots,h$, it denotes the function
\[
(f_1,\dots,f_h)  : \mathbb{R}^{n \times p} \to \mathbb{R}^{mh \times p}, \quad  X \mapsto \begin{bmatrix} f_1(X) \\ \vdots \\ f_h(X) \end{bmatrix}.
\]

The function $\SoftMax : \mathbb{R}^n \to \mathbb{R}^n$  takes a vector $x = (x_1, \dots, x_n) \in \mathbb{R}^n$ and outputs a probability vector of the same dimension,
\[
\SoftMax(x) \coloneqq \biggl(\frac{e^{x_1}}{\sum_{i=1}^n e^{x_i}}, \dots, \frac{e^{x_n}}{\sum_{i=1}^n e^{x_i}}\biggr).
\]
When $\SoftMax$ is applied to a matrix $X \in \mathbb{R}^{n \times p}$, it is applied columnwise to each of the $p$ columns of $X$. So $\SoftMax : \mathbb{R}^{n \times p} \to \mathbb{R}^{n \times p}$.

Although we will write $\mathbb{R}$ throughout to avoid clutter, we will allow for $-\infty$ in the argument of our functions on occasion, which will be clearly indicated. Note that $\SoftMax(x)_i = 0$ if $x_i = -\infty$.

\subsection{Feed forward neural network}\label{sec:nn}

The rectified linear unit $\ReLU : \mathbb{R} \to \mathbb{R}$ is defined by $\ReLU(x) = \max(x,0) \eqqcolon x^+$ and extended coordinatewise to vectors in $\mathbb{R}^n$ or matrices in $\mathbb{R}^{n \times p}$. We also introduce the shorthand $x^- \coloneqq \ReLU(-x)$. Clearly, $X = X^+ - X^-$ for any $X \in \mathbb{R}^{n \times p}$.

An $l$-layer feed forward neural network is a map $\varphi:\mathbb{R}^n \to \mathbb{R}^{n_{l+1}}$ defined by a composition:
\[
\varphi(x) = A_{l+1} \sigma_{l} A_{l}  \cdots  \sigma_2 A_2\sigma_1 A_1 x + b_{l+1}
\]
for any input $x \in \mathbb{R}^n$, weight matrix  $A_i \in \mathbb{R}^{n_i\times n_{i-1}}$, with $n_0 = n$, $\sigma_i(x) \coloneqq \sigma(x +b_i)$, 
with $b_i \in \mathbb{R}^{n_i}$ the bias vector, and $\sigma : \mathbb{R} \to \mathbb{R}$ the activation function, applied coordinatewise. In this article, we set $\sigma = \ReLU$ throughout. To avoid clutter we omit the  $\circ$ for function composition within a feed forward neural network unless necessary for emphasis, i.e., we will usually write $A \sigma B$ instead of $A \circ \sigma \circ B$. When $\varphi$ is applied to a matrix $X  \in \mathbb{R}^{n \times p}$,  it is always applied columnwise to each of the $p$ columns of $X$. So $\varphi : \mathbb{R}^{n \times p} \to \mathbb{R}^{n_{l+1} \times p}$. We will also drop the ``feed forward'' henceforth since all neural networks that appear in our article are feed forward ones.

\subsection{Attention}\label{sec:att}

The  \emph{attention module} is known by a variety of other names, usually a combination of attention/self-attention module/mechanism, and  usually represented as flow charts as in Figure~\ref{fig:attention}.

\begin{figure}[htb]
    \centering
     \begin{subfigure}[b]{0.35\textwidth}
        \centering
        \begin{tikzpicture}[scale=0.8,
            every node/.append style={thick,rounded corners=2pt}]
            \node (q) at (.5,0) {$Q$};
            \node (k) at (1.5,0) {$K$};
            \node (v) at (2.5,0) {$V$};
        
            \node[draw,fill=myblue!20] (matmul1) at (1,1) {MatMul};
            \node[draw,fill=myred!40] (scale) at (1,2) {Scale};
            \node[draw,fill=mygreen!20] (SoftMax) at (1,3) {SoftMax};
            \node[draw, fill=blue!20, minimum width=1.75cm] (matmul2) at (1.75,4) {MatMul};
        
            \path[->,thick]
                (matmul1) edge (scale)
                (scale) edge (SoftMax)
                (q) edge (.5,.7)
                (k) edge (1.5,.7)
                (v) edge (2.5,3.7)
                (SoftMax) edge (1,3.7);
        \end{tikzpicture}
        \caption{SoftMax attention}
    \end{subfigure}
\quad
    \begin{subfigure}[b]{0.35\textwidth}
        \centering
        \begin{tikzpicture}[scale=0.8,
            every node/.append style={thick,rounded corners=2pt}]
            \node (q) at (.5,0) {$Q$};
            \node (k) at (1.5,0) {$K$};
            \node (v) at (2.5,0) {$V$};
        
            \node[draw,fill=myblue!20] (matmul1) at (1,1) {MatMul};
            \node[draw,fill=myred!40] (scale) at (1,2) {Scale};
            \node[draw,fill=mygreen!20] (ReLU) at (1,3) {ReLU};
            \node[draw, fill=blue!20, minimum width=1.75cm] (matmul2) at (1.75,4) {MatMul};
        
            \path[->,thick]
                (matmul1) edge (scale)
                (scale) edge (ReLU)
                (q) edge (.5,.7)
                (k) edge (1.5,.7)
                (v) edge (2.5,3.7)
                (SoftMax) edge (1,3.7);
        \end{tikzpicture}
        \caption{ReLU attention}
    \end{subfigure}
    \caption{Attention module as flow chart}
    \label{fig:attention}
\end{figure}
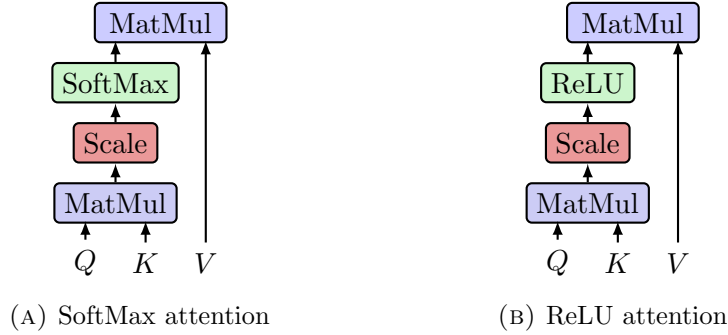

Mathematically, it is a map $\alpha: \mathbb{R}^{n \times p} \to \mathbb{R}^{m \times p}$,
\begin{equation}\label{eq:att}
\alpha(X) \coloneqq V(X)  \SoftMax\bigl(K(X)^\tp  Q(X)\bigr) ,
\end{equation}
where $Q :  \mathbb{R}^{n \times p} \to \mathbb{R}^{d \times p}$, $K :  \mathbb{R}^{n \times p} \to \mathbb{R}^{d \times p}$, $V :  \mathbb{R}^{n \times p} \to \mathbb{R}^{m \times p}$ are \emph{linear layers}, i.e., given by affine maps
\begin{equation}\label{eq:QKV}
    Q(X)  = A_Q X + B_Q, \quad
    K(X)  = A_K X + B_K, \quad
    V(X)  = A_V X + B_V,
\end{equation}
with weight matrices $A_Q, A_K \in \mathbb{R}^{d \times n}$, $A_V \in \mathbb{R}^{m \times n}$, and
bias matrices $B_Q, B_K \in \mathbb{R}^{d \times p}$, $B_V \in \mathbb{R}^{m \times p}$. Here we have used the more general \emph{affine} form of these linear layers as attention modules are implemented in practice,\footnote{\url{https://pytorch.org/docs/stable/generated/torch.nn.MultiheadAttention.html\#torch.nn.MultiheadAttention}; the affine form is obtained by setting \texttt{add\_bias\_kv = True}.} as opposed to the \emph{linear} form in  \cite{vaswani2017attention} where the biases are set to zero. The $\SoftMax$ in \eqref{eq:att} is applied columnwise and outputs a $p \times p$ matrix. 

The map $\alpha$ implements the mechanism of taking a query and a set of key--value pairs to an output. Interpreted in this way, the input $X \in \mathbb{R}^{n \times p}$ is a data sequence of length $p$, with each data point $x_i \in \mathbb{R}^n$, $i =1,\dots, p$. The columns of  $Q(X)$ and $K(X)$ represent queries and keys respectively --- note that these are vectors in $\mathbb{R}^d$ and $d$ is generally much smaller than $m$ or $n$. The columns of $V(X)$ represent values.

More generally, a multihead or $h$-headed attention module is a map $\alpha: \mathbb{R}^{n \times p} \to  \mathbb{R}^{mh \times p}$ given by
\begin{equation}\label{eq:h}
\alpha(X) = (\alpha_1(X), \dots, \alpha_h(X)) 
\end{equation}
where $\alpha_i: \mathbb{R}^{n \times p} \to  \mathbb{R}^{m \times p}$ are attention modules as in \eqref{eq:att}, $i = 1, \dots, h$. The reader is reminded of our convention in Section~\ref{sec:notate}: parentheses denote column, which is why in our constructions we will often the phrase ``stacking $\alpha_1,\dots,\alpha_h$ to obtain $\alpha$'' to mean \eqref{eq:h}.

\subsection{Encoder}\label{sec:enc}
 
An \emph{encoder block}, or more precisely a $h$-head encoder block, is a map $\varepsilon : \mathbb{R}^{n \times p} \to \mathbb{R}^{n_{l+1} \times p}$ obtained by composing the output of a $h$-head attention module $\alpha: \mathbb{R}^{n \times p} \to  \mathbb{R}^{mh \times p}$, with an $l$-layer $\ReLU$-neural network $\varphi : \mathbb{R}^{mh \times p} \to \mathbb{R}^{n_{l+1} \times p}$,
\begin{equation}\label{eq:encblk}
\varepsilon = \varphi \circ \alpha.
\end{equation}
More generally, an \emph{encoder} or $t$-layer encoder, $\varepsilon_t : \mathbb{R}^{n \times p} \to \mathbb{R}^{n_{t+1} \times p}$ is obtained by composing $t$ encoder blocks, i.e.,
\begin{equation}\label{eq:enc}
\varepsilon_t = \varphi_t \circ \alpha_t \circ \varphi_{t-1} \circ \alpha_{t-1} \circ \dots \circ\varphi_1 \circ \alpha_1,
\end{equation}
where $\varphi_i: \mathbb{R}^{m_i \times p} \to \mathbb{R}^{n_{i+1} \times p}$ are  neural networks and $\alpha_i: \mathbb{R}^{n_i \times p} \to \mathbb{R}^{m_i \times p}$ are attention modules, $i=1,\dots,t$, $n_1 = n$. In Figure~\ref{fig:trans}, the encoder is the part enclosed within the red dash lines on the left.  The structure  in \eqref{eq:enc} appears to require alternate compositions of attention modules and  neural networks but one may skip some or all of the $\varphi_i$'s. The reason is that we may choose these $\varphi_i$'s to be an identity map, which can be represented as a one-layer  neural network as $x = \ReLU(x) - \ReLU(-x)$.

While we allow the neural networks appearing in \eqref{eq:enc}  to have multiple hidden layers, the original proposed model in \cite{vaswani2017attention} requires that they be single-layer. We will show in Lemma~\ref{lem:1layer} that these are in fact equivalent: Any encoder of the form \eqref{eq:enc} may be written as one where all $\varphi_i$'s have only one hidden layer, but at the expense of a larger $t$.

\subsection{Masked attention}\label{sec:mask}

In many applications of transformers, particularly large language models, the data is of a sequential nature. So the function $f$ we want to learn or approximate is expected to be \emph{autoregressive} \cite{vaswani2017attention}, i.e.,  $f: \mathbb{R}^{n \times p} \to \mathbb{R}^{m \times p}$ takes the form
\begin{equation}\label{eq:auto}
[x_1, \dots, x_p] \mapsto [f_1(x_1), f_2(x_1,x_2), \dots, f_p(x_1,\dots,x_p)].
\end{equation}
In other words $f_j : \mathbb{R}^{n \times j} \to \mathbb{R}^m$ depends only on the first $j$ columns $x_1, \dots, x_j$, $j=1,\dots,p$. In general $f$ will be nonlinear, but when $f$ is linear, then this simply means it is given by an upper triangular matrix. So an autoregressive function may be viewed as a nonlinear generalization of an upper triangular matrix.

To achieve this property in attention module, we define the function $\mask: \mathbb{R}^{p \times p} \to \mathbb{R}^{p \times p}$ by
\[
\mask(X)_{ij} =
\begin{cases}
         x_{ij} & \text{if } i \le j,\\
        -\infty & \text{if } i > j.
\end{cases}
\]
A \emph{masked attention} module is then given by
\begin{equation}\label{eq:matt}
\beta(X) = V(X) \SoftMax\bigl(\mask (K(X)^\tp Q(X))\bigr).
\end{equation}
It is easy to check that a masked attention module is always autoregressive. 

\subsection{Decoder}\label{sec:dec}

A \emph{decoder block} is the analogue of an encoder block where we have a masked attention in \eqref{eq:encblk}:
\begin{equation}\label{eq:decblk}
\delta = \varphi \circ \beta.
\end{equation}
We may also replace any or all of the $\alpha_i$'s in \eqref{eq:enc} by masked versions $\beta_i$'s. If we replace all, then the resulting map
\begin{equation}\label{eq:dec}
\delta_t = \varphi_t \circ \beta_t \circ \varphi_{t-1} \circ \beta_{t-1} \circ \dots \circ\varphi_1 \circ \beta_1,
\end{equation}
is autoregressive but more generally we will just selectively replace some $\alpha_i$'s with $\beta_i$'s. We call the resulting map a \emph{decoder}. Note that the part enclosed within red dash lines in the right-half of Figure~\ref{fig:trans} is not quite a decoder as it takes a feed from the left-half; instead it is an \emph{encoder--decoder}, as we will discuss next.

\subsection{Encoder--decoder attention}\label{sec:encdec}

The multihead attention in the right-half of Figure~\ref{fig:trans} accepts a feed from outside the red dash box. When used in this manner, it is called  an \emph{encoder--decoder attention module} \cite{vaswani2017attention}, as it permits one to use queries from the decoder, but keys and values from the encoder. Mathematically, this is a map $\gamma: \mathbb{R}^{n\times p} \times \mathbb{R}^{r \times p} \to \mathbb{R}^{m\times p}$,
\begin{equation}\label{eq:encdec}
\gamma(X,Y) \coloneqq V(X)  \SoftMax\bigl(K(X)^\tp Q(Y)\bigr) ,
\end{equation}
where $Q, K, V$ are as in \eqref{eq:QKV} but while $K,V$ are functions of $X$, $Q$ is now a function of $Y$. The independent matrix variables $X$ and $Y$ take values in $\mathbb{R}^{n \times p}$ and $\mathbb{R}^{r \times p}$  respectively. As a result we have to adjust the dimensions of the weight matrices slightly: $A_Q \in \mathbb{R}^{d \times r}$, $A_K \in \mathbb{R}^{d \times n}$, $A_V \in \mathbb{R}^{m \times n}$. The encoder--decoder attention is \emph{partially autoregressive}, i.e., autoregressive in $Y$ but not in $X$, taking the form
\[
(X, [y_1, \dots, y_p])
  \mapsto  [f_1(X, y_1), f_2(X, y_2), \dots, f_p(X, y_1,\dots,y_p)].
\]

\subsection{Transformer}\label{sec:trans}

An \emph{encoder--decoder block} $\tau :  \mathbb{R}^{n\times p} \times \mathbb{R}^{r \times p} \to \mathbb{R}^{n_{l+1} \times p}$ is defined by a multihead masked attention module $\beta$, a multihead encoder--decoder attention module $\gamma$, and a  neural network $\varphi$, via
\[
\tau(X, Y) = \varphi\bigl(\gamma(X, \beta(Y))\bigr).
\]
An $(s + t)$-layer \emph{encoder--decoder} is then constructed from an  $s$-layer encoder $\varepsilon_s$, and $t$ encoder--decoder blocks given by $\beta_1, \gamma_1, \varphi_1, \dots, \beta_t, \gamma_t, \varphi_t$. We define $\tau_i$ recursively as
\[
\tau_i(X, Y) = \varphi_i\bigl(\gamma_i\bigl(\varepsilon_s(X), \beta_i(\tau_{i-1}(X, Y))\bigr)\bigr)
\]
for $i =1,\dots,t$, $\tau_0(X, Y) = Y$. We call $\tau_t$ the encoder--decoder. For all mathematical intents and purposes, $\tau_t$ is the \emph{transformer}. As we will see in Sections~\ref{sec:norm} and \ref{sec:misc}, the other components in Figure~\ref{fig:trans} or \cite[Figure~1]{vaswani2017attention} are extraneous to the operation of a transformer.

We stress that the word ``transformer'' is sometimes used to refer to just the encoder or the decoder alone. We choose to make the distinction in our article but many do not. For example, Google's BERT \cite{bert}, for Bidirectional Encoder Representations from Transformers, is an encoder whereas OpenAI's GPT \cite{gpt3}, for Generative Pretrained Transformer, is a decoder.

\subsection{ReLU-transformer}

The definitions in Sections~\ref{sec:nn}--\ref{sec:trans} are faithful mathematical transcriptions of components as described in Vaswani et al.\ original article \cite{vaswani2017attention}. In this section we take a small departure --- replacing every occurrence of $\SoftMax$ with $\ReLU$ to obtain what is called a $\ReLU$-transformer. This is not new either but proposed and studied in \cite{bai2023Transformers, wortsman2023replacing}.

We begin by defining $\ReLU$-attention modules. They have the same structures as \eqref{eq:att}, \eqref{eq:matt}, \eqref{eq:encdec} except that $\SoftMax$ is replaced by $\ReLU$, i.e.,
\begin{equation}\label{eq:att1}
\begin{aligned}
\alpha(X) &= V(X) \ReLU\bigl(K(X)^\tp Q(X)\bigr),\\
\beta(X) &= V(X) \ReLU\bigl(\mask(K(X)^\tp Q(X))\bigr),\\
\gamma(X,Y) &= V(X)  \ReLU\bigl(K(X)^\tp Q(Y)\bigr).
\end{aligned}
\end{equation}
An encoder, decoder, or encoder--decoder constructed out of such $\ReLU$-attention modules will be called a $\ReLU$-encoder, $\ReLU$-decoder, or $\ReLU$-encoder--decoder respectively. In particular, a $\ReLU$-transformer is, for all mathematical intents and purposes, a $\ReLU$-encoder--decoder.

These $\ReLU$-activated variants are essentially ``unsmoothed'' versions of their smooth $\SoftMax$-activated cousins in Sections~\ref{sec:nn}--\ref{sec:trans}. We may easily revert to the smooth versions by a simple smoothing process --- replace all $\ReLU$-activated attentions by the original $\SoftMax$-activated ones (but the  neural networks would remain $\ReLU$-activated). 

$\ReLU$-transformers work naturally with our claims and proofs in Section~\ref{sec:equiv}. Nevertheless, even in practice $\ReLU$-transformers can have desirable, possibly superior,  features compared to the original $\SoftMax$-transformers: investigations in \cite{wortsman2023replacing} provided extensive empirical evidence that substituting $\SoftMax$ with $\ReLU$ causes no noticeable loss and occasionally even affords a slight gain in performance across both language and vision tasks; it is also easier to explain the in-context-learning capability of $\ReLU$-transformers \cite{bai2023Transformers}.

More generally, the use of alternative activations in a transformer is a common practice. There are various reasons to replace $\SoftMax$, one of which is to avoid the considerable training cost associated with the use of $\SoftMax$ activation. In \cite{NEURIPS2021_b1d10e7b},  $\SoftMax$ is replaced with a Gaussian kernel; in \cite{koohpayegani2024sima}, only the normalization part of $\SoftMax$ is kept; in \cite{he2023deep}, it is shown that an activation does not need to map into the probability simplex. Linearized attentions are used in  \cite{qin2021cosformer}, and sparse attentions in \cite{peters2019sparse}; these are intended primarily to accelerate the $\SoftMax$ operator but they have other features too.

\subsection{Layer normalization and residual connection}\label{sec:norm}

Comparing our Figure~\ref{fig:trans} and  \cite[Figure~1]{vaswani2017attention}, one might notice that we have omitted the ``add \& norm'' layers.

 The ``add'' step, also called residual connection \cite{he2016deep}, may be easily included in our analysis --- all our results and proofs in Section~\ref{sec:equiv} hold verbatim with the inclusion of residual connection. For an encoder block $\varepsilon : \mathbb{R}^{n \times p} \to \mathbb{R}^{n \times p}$, a residual connection simply means adding the identity map $\iota  : \mathbb{R}^{n \times p} \to \mathbb{R}^{n \times p}$, $X \mapsto X$, i.e.,
\[
\varepsilon + \iota : \mathbb{R}^{n \times p} \to \mathbb{R}^{n \times p}, \quad X \mapsto \varepsilon(X) + X,
\]
and likewise for a decoder block $\delta : \mathbb{R}^{n \times p} \to \mathbb{R}^{n \times p}$. For an encoder--decoder block $\tau : \mathbb{R}^{n \times p} \times \mathbb{R}^{r \times p}\to \mathbb{R}^{r \times p}$, a residual connection simply means adding the projection map $\pi : \mathbb{R}^{n \times p} \times \mathbb{R}^{r \times p}\to \mathbb{R}^{r \times p}$, $(X,Y) \mapsto Y$, i.e.,
\[
\tau + \pi : \mathbb{R}^{n \times p} \times \mathbb{R}^{r \times p}\to \mathbb{R}^{r \times p}, \quad (X, Y) \mapsto \tau(X, Y) + Y.
\]
As will be clear from the proofs in Section~\ref{sec:equiv}, all results therein hold with or without residual connection.

The ``norm'' step, also called layer normalization \cite{ba2016layer} refers to statistical standardization, i.e., mean centering and scaling by standard deviation of each column vector in $X$. This is an ubiquitous process routinely performed in just about any procedure involving any data for practical reasons. But this innocuous process introduces additional nonlinearity that does not fit in our framework. 

We do not consider either of these critical to the workings of a transformer. They are by no means unique and may be easily replaced with other data standardization process, as shown in \cite{he2023deep}.

\subsection{Miscellany}\label{sec:misc}

The ``input/output embedding'' and ``position embedding''  in Figure~\ref{fig:trans} convert sentences or images (or whatever real-world entity the transformer is used for) to an input in $\mathbb{R}^{n \times p} $; the ``linear layer'' and ``SoftMax'' in the right half assign probability values to the output. These are just auxiliary components necessary in any situation involving human-generated input or requiring human-interpretable output. They are common to all practical AI models and we do not regard them as part of the transformer architecture.

\section{Splines}\label{sec:spline}

This section covers the salient aspects of splines relevant for us. We write $\mathbb{R}[x_1, \dots, x_n]$ for the ring of polynomials with real coefficients in variables $ (x_1,\dots,x_n) \eqqcolon x$ and  $\mathbb{R}[x_{11},\dots,x_{np}]$  for that in $ (x_{ij})_{i,j=1}^{n,p} \eqqcolon X$.

Splines have a rich history and a vast literature in applied and computational mathematics, this being precisely the reason we chose them as our platform to understand a new technology like the transformer. Mathematical \emph{splines}, as opposed to the mechanical ones used by draftsmen and shipbuilders, were first named in \cite{SchoenbergPNAS}. A one-line summary of its early history, with many regretful omissions, is that univariate splines were first proposed in \cite{Schoenberg1946}, multivariate splines in \cite{birkhoff1960smooth}, B-Splines in \cite{curry1966polya}, and box splines in \cite{BoxSpline}. 

An important departure of our discussion of splines in this article is that we will not concern ourselves with differentiability, avoiding the usual efforts to ensure that a piecewise-defined function is $C^r$ at points where the different pieces meet. The reason is simple: our results in the next section will show that every continuous spline is a $\ReLU$-transformer (and vice versa) and when presented as such, there is a straightforward and natural way to smooth a spline to any desired degree-of-smoothness $r$, namely, by replacing $\ReLU$ with a $C^r$-activation. So there is no need for us to even introduce the notions of knots, tangential continuity, curvature continuity, etc. Indeed, viewed in this manner, the transformer with its $\SoftMax$ activation is the first example of a ``$C^\infty$-spline'' --- an impossible object in classical constructions of splines as the degree-of-smoothness of a spline can never exceed the degree of its polynomial pieces.

\subsection{Scalar-valued splines}\label{sec:sca}

In its simplest form a spline is a piecewise-polynomial real-valued function $f : \mathbb{R}^n \to \mathbb{R}$ defined over a partition of its domain $\mathbb{R}^n$. The classical and most basic partition is a triangulation, i.e., a subdivision into $n$-dimensional simplices whose union is $\mathbb{R}^n$ and intersecting only along faces; more generally one may also use convex polytopes in place of simplices  \cite{de1978practical, Chui, Handbook}. We will need a slightly more sophisticated partition called a semialgebraic partition \cite{dipasquale2017, dipasquale2020, shek}. For any $b \in \mathbb{N}$, let
\begin{equation}\label{eq:Phi}
\Theta_b \coloneqq \bigl\{ \theta: \{1, \dots, b\} \to \{1, 0, -1\} \bigr\},
\end{equation}
a finite set of size $3^b$. Note that this is really just the set of ternary numerals with $b$ (ternary) bits.
\begin{definition}[Partition]\label{def:par}
Any $\pi_1, \dots, \pi_b \in \mathbb{R}[x_1, \dots, x_n]$ induces a sign partition of $\mathbb{R}^n$ via
\[
\Pi_{\theta} \coloneqq \{x \in\mathbb{R}^n : \sgn(\pi_i(x)) = \theta(i), \; i = 1,\dots,b\}.
\]
Then $\{\Pi_{\theta} : \theta \in \Theta_b \}$ is a partition of $\mathbb{R}^n$, the \emph{semialgebraic partition} induced by $\pi_1,\dots,\pi_b$.
\end{definition}
Note that the domain of $\theta$  in \eqref{eq:Phi} merely serves as a placeholder for any $b$-element set and does not need to be $\{1,\dots,b\}$. Indeed we will usually write $\theta: \{\pi_1, \dots, \pi_b\} \to \{1, 0, -1\}$ to emphasize that it is an index for the partition induced by $\pi_1,\dots,\pi_b$. Any triangulation or partition into polytopes can be obtained by choosing appropriate linear polynomials  $\pi_1,\dots,\pi_b$ so Definition~\ref{def:par} generalizes the basic one that requires partition to be piecewise linear.
\begin{definition}[Spline]\label{def:spline}
Let $\{\Pi_{\theta} : \theta \in \Theta_b \}$ be the semialgebraic partition induced by $\pi_1, \dots, \pi_b \in \mathbb{R}[x_1, \dots, x_n]$. A continuous function $f: \mathbb{R}^n \to \mathbb{R}$ is a \emph{polynomial spline} of degree $k$ if  for each $i =1,\dots,b$,
\begin{enumerate}[label=\upshape{(\roman*)}, nosep]
\item $\pi_i$ has degree not more than $k$;
\item if $\Pi_\theta \ne \varnothing$, then $f$ restricts to a polynomial of degree not more than $k$ on $\Pi_\theta$, i.e.,
$f(x) = \xi_\theta(x)$ for all $x \in \Pi_\theta$, for some $\xi_\theta  \in \mathbb{R}[x_1, \dots, x_n]$ of degree not more than $k$.
\end{enumerate}
\end{definition}
Henceforth, ``spline'' will mean ``polynomial spline,'' ``degree-$k$'' will mean ``degree not more  than $k$,'' and ``partition'' will mean ``semialgebraic partition.'' The small cases $k =1,2,3,5$ are customarily called linear, quadratic, cubic, and quintic splines respectively. The standard notation for the set of all $r$-times differentiable degree-$k$ splines with partition induced by  $\pi_1, \dots, \pi_b$ is  $S_k^r(\pi_1,\dots,\pi_b)$ but since we will only need the case $r = 0$ and splines as defined in Definition~\ref{def:spline} are always continuous, we may drop the superscript $r$.

Observe that $S_k(\pi_1,\dots,\pi_b)$ is a finite-dimensional real vector space. So it is straightforward to extend Definition~\ref{def:spline} to $\mathbb{V}$-valued splines $f : \mathbb{R}^n \to \mathbb{V}$  for any finite-dimensional real vector space $\mathbb{V}$ using tensor product, namely, they are simply elements of $S_k(\pi_1,\dots,\pi_b) \otimes \mathbb{V}$ \cite[Example~4.30]{acta}. For the benefit of readers unfamiliar with tensor product constructions, we go over this below in a concrete manner for $\mathbb{V} = \mathbb{R}^n$ and $\mathbb{R}^{n \times p}$.

\subsection{Vector-valued splines}\label{sec:vec}

A vector-valued degree-$k$ spline $f : \mathbb{R}^n \to \mathbb{R}^m$ is given by
\[
f(x) =   \sum_{i=1}^m f_i(x) e_i \quad\text{for all } x \in \mathbb{R}^n,
\]
where $f_1,\dots,f_m \in S_k(\pi_1,\dots,\pi_b)$ and $e_1,\dots,e_m \in \mathbb{R}^m$ are the standard basis vectors. This is equivalent to requiring $f$ be a degree-$k$ spline coordinatewise, i.e., $f = (f_1,\dots,f_m)$ where $f_1,\dots,f_m \in S_k(\pi_1,\dots,\pi_b)$.

Traditionally, vector-valued splines are the most important class of splines for practical applications. Special cases include spline curves ($n = 1$, $m = 2$ or $3$) and spline surfaces  ($n = 2$, $m = 2$ or $3$), used to parameterize curves and surfaces that pass near a collection of given data points. These are of fundamental importance in computer graphics and computer-aided design  \cite{HCAD, Shikin}. 

\subsection{Matrix-valued splines}

In this case we are interested in splines that are not just matrix-valued but also matrix-variate. One nice feature with our treatment of splines in Section~\ref{sec:sca} is that we can define matrix-variate splines over $\mathbb{R}^{n \times p}$ by simply replacing all occurrences of $\mathbb{R}[x_1, \dots, x_n]$ with $\mathbb{R}[x_{11},\dots,x_{np}]$. A matrix-valued degree-$k$ spline $f : \mathbb{R}^{n \times p} \to \mathbb{R}^{m \times p}$ is then given by
\begin{equation}\label{eq:mat}
f(X)  =   \sum_{i=1}^m \sum_{j=1}^p f_{ij}(X) E_{ij} \quad\text{for all } X \in \mathbb{R}^{n \times p},
\end{equation}
where $f_{ij} \in S_k(\pi_1,\dots,\pi_b)$ and $E_{ij} \in \mathbb{R}^{m \times p}$, $i=1,\dots,m$, $j=1,\dots,p$. Here $E_{ij}$ is the standard basis matrix with one in $(i,j)$th entry and zeros everywhere else.  Again, an alternative but equivalent way to define them would be in a coordinatewise fashion, i.e., $f = (f_{ij})_{i,j=1}^{m,p}$ where $f_{ij} \in S_k(\pi_1,\dots,\pi_b)$, $i=1,\dots,m$, $j=1,\dots,p$. Note that $p=1$ reduces to the case in Section~\ref{sec:vec}.

\subsection{Pierce--Birkhoff conjecture}

Garrett Birkhoff, likely the person first to realize the importance of splines in applications though his consulting work \cite{Young}, also posed one of the last remaining open problems about splines \cite{BP}.
\begin{conjecture}[Pierce--Birkhoff]\label{conj:PB}
For every spline $f : \mathbb{R}^n \to \mathbb{R}$, there exists a finite set of polynomials $\xi_{ij} \in \mathbb{R}[x_1,\dots,x_n]$, $i =1,\dots,m$, $j=1,\dots,p$ such that 
\begin{equation}\label{eq:PB}
f=\max_{i=1,\dots,m}\min_{j=1,\dots,p} \xi_{ij}.
\end{equation}
\end{conjecture}
This conjecture is known to be true for $n = 1$ and $2$ but is open for all higher dimensions \cite{Mahe1984}. Our results in Section~\ref{sec:s=t} will be established on the assumption that the Pierce--Birkhoff conjecture holds true for all $n$, given that there is significant evidence \cite{PB1,PB2,PB4} for its validity. 

The kind of functions on the right of \eqref{eq:PB} we will call \emph{max-definable functions} in the variables $x_1,\dots,x_n$.  These are functions $f : \mathbb{R}^n \to \mathbb{R}$  generated by  $1, x_1, \dots, x_n$ under three binary operations: addition $(x,y) \mapsto x+y$, multiplication $(x,y) \mapsto x \cdot y$, maximization $(x,y) \mapsto \max(x,y)$; and scalar multiplication $x \mapsto \lambda x$ by $\lambda \in \mathbb{R}$. Note that minimization comes for free as $\min(x,y) \coloneqq -\max(-x,-y)$.  Using the identity
$xy^+ = \max\bigl(\min(xy, x^2y+y), \min(0, -x^2y-y)\bigr)$,
any max-definable functions can be reduced to the form $\max_{i=1,\dots,m}\min_{j=1,\dots,p} \xi_{ij}$ with  $\xi_{ij} \in \mathbb{R}[x_1,\dots,x_n]$  \cite{HI}. The notion may be easily extended to matrix-variate, matrix-valued functions $f : \mathbb{R}^{n \times p} \to \mathbb{R}^{m \times p}$ coordinatewise, i.e., by requiring that each $f_{ij} :  \mathbb{R}^{n \times p} \to \mathbb{R}$ be a max-definable function in the variables $x_{11}, x_{12},\dots,x_{np}$.

Clearly, the set of max-definable functions is contained within the set of splines. Pierce--Birkhoff conjecture states that the two sets are equal. Both are examples of an ``$f$-ring'' as defined in \cite{BP}, now christened ``Pierce--Birkhoff ring'' after the two authors. If we drop multiplication from the list of binary operations generating the max-definable functions, the resulting algebraic object is the renown max-plus algebra or tropical semiring \cite{maclagan2015}.

\section{Equivalence of splines and transformers}\label{sec:equiv}

We will show that every component of the transformer defined in Section~\ref{sec:math} is a spline ---  neural network, attention module, masked attention module, encoder block, decoder block, encoder, decoder, encoder--decoder --- so long as they are $\ReLU$-activated. More importantly, if Conjecture~\ref{conj:PB} is true, then the converse also holds in the sense that every spline is an encoder. The equivalence between $\ReLU$-activated feed-forward neural networks and linear splines is well-known \cite{arora2018understanding}. The other equivalences will be established below. Henceforth we will assume $\ReLU$-activation throughout this section and will not specify this unless necessary for emphasis.

\subsection{Transformers are splines}

We will first remind readers of the main result in \cite{arora2018understanding} establishing equivalence between  neural networks and linear splines.
\begin{theorem}[Arora--Basu--Mianjy--Mukherjee]\label{thm:relulinear}
    Every neural network  $\varphi : \mathbb{R}^n \to \mathbb{R}$ is a linear spline, and every linear spline $\ell : \mathbb{R}^n \to \mathbb{R}$ can be represented by a neural network with at most $\lceil\log _2(n+1)\rceil+1$ depth.
\end{theorem}

Compositions of spline functions are by-and-large uncommon in the literature for reasons mentioned in the beginning --- one usually combines splines by taking sums or linear combinations. Matrix-valued splines also appear to be somewhat of a rarity in the literature. Consequently we are unable to find a reference for what ought to be a fairly standard result about degrees under composition and matrix multiplication, which we state and prove below.
\begin{lemma}\label{lem:comp}
\begin{enumerate}[label=\upshape{(\roman*)}, nosep]
\item\label{it:comp} Let $g : \mathbb{R}^n \to \mathbb{R}^m$ be a spline of degree $k$ and $f : \mathbb{R}^m \to \mathbb{R}^p$ a spline of degree $k'$. Then $f \circ g$ is a spline of degree $kk'$.

\item\label{it:mult} Let $f: \mathbb{R}^{r \times s}  \to \mathbb{R}^{m \times n}$ and $g:  \mathbb{R}^{r \times s} \to \mathbb{R}^{n \times p}$ be splines of degrees $k$ and $k'$. Then $fg :  \mathbb{R}^{r \times s}  \to \mathbb{R}^{m \times p}$, $X \mapsto f(X)g(X)$, is a spline of degree $k + k'$.
\end{enumerate}
\end{lemma}
\begin{proof}
We first assume that $p = 1$, i.e., $f : \mathbb{R}^m \to \mathbb{R}$ is a  spline of degree $k'$. For a degree-$k$ spline $g = (g_1, \dots, g_m) : \mathbb{R}^n \to \mathbb{R}^m$, we claim that the composition $f\circ g$ is a spline of degree at most $kk'$.

A partition induced by any $\pi_1, \dots, \pi_b$ can be \emph{refined} to $\pi_1, \dots, \pi_b, \pi_{b+1},\dots,\pi_{b + c}$ by adding finitely many polynomials. Any spline in $S_k(\pi_1,\dots,\pi_b)$ is also a spline in $ S_k(\pi_1,\dots,\pi_{b + c})$. By passing through such refinements, we may assume that  $g_1,\dots,g_m$ are defined over a common partition. So let $g_1,\dots,g_m \in S_k(\pi_1, \dots, \pi_b)$ with
\[
g_i(x) = \xi_{i, \theta}(x) \quad\text{for } x \in \Pi_\theta, \quad \theta \in \Theta_b
\]
where $\Theta_b = \bigl\{ \theta : \{\pi_1,\dots,\pi_b\} \to \{-1,0,1\} \bigr\}$.
Let $f \in S_k (\rho_1, \dots, \rho_c)$ with
\[
f(x) = \zeta_\phi(x) \quad\text{for } x \in \Pi_\phi, \quad \phi \in \Phi_c
\]
where $\Phi_c = \bigl\{ \phi : \{\rho_1,\dots,\rho_c\} \to \{-1,0,1\} \bigr\}$.
Let $L \coloneqq \{\pi_1,\dots,\pi_b\}$ and
\[
M \coloneqq L \cup \{ \rho_j \circ (\xi_{1, \theta}, \dots, \xi_{m, \theta}) :  j = 1,\dots, c, \; \theta \in \Theta_b\}. 
\]
Any $\phi: M \to \{1, 0, -1\}$ can be restricted to $L$, giving $\phi\rvert_L: L \to \{1, 0, -1\}$. Let
\[
H \coloneqq \{\rho_j\circ(\xi_{1, \phi\rvert_L}, \dots, \xi_{m, \phi\rvert_L}) :  j=1,\dots, c\} \subseteq L.
\]
Then $\phi$ can also be restricted to $H$, giving $\phi\rvert_H:H \to \{1, 0, -1\}$. For any nonempty $\Pi_\phi$, we have
\[
f\circ g (x) = \zeta_{\phi\rvert_H} \circ (\xi_{1, \phi\rvert_L}, \dots, \xi_{m, \phi\rvert_L}) (x)
\]
for $x \in \Pi_\phi$ where $\phi \in \Phi_c$. So $f \circ g \in S_{kk'} (M)$. This shows \ref{it:comp} for $p = 1$. For general $p$, we may again assume, by passing through a refinement if necessary, that $f_1,\dots,f_p$ share a common partition, we then apply the same argument coordinatewise.

We then deduce \ref{it:mult} from \ref{it:comp}, by composing the spline $(f, g)$ with the polynomial (and therefore spline) function $(X, Y) \to XY$.
\end{proof}

With the ground work laid in Section~\ref{sec:math}, i.e., having the components of a transformer rigorously defined, it becomes relatively straightforward to show that these components are all splines.
\begin{theorem}[Components of a transformer as splines]\label{thm:main1}\hfill
\begin{enumerate}[label=\upshape{(\roman*)}, nosep]
\item An attention module is a cubic spline.
\item\label{it:mask} A masked attention module is a cubic spline.
\item\label{it:edatt} An encoder--decoder attention module is a cubic spline.
\item An encoder block is a cubic spline. 
\item\label{it:decblk} A decoder block is a cubic spline. 
\item\label{it:edblk} An encoder--decoder block is a quintic spline. 
\item\label{it:enc} A $t$-layer encoder is a spline of degree $3^t$. 
\item\label{it:dec} A $t$-layer decoder is a spline of degree $3^t$. 
\item\label{it:encdec} An encoder--decoder with $s$-layer of encoder blocks and $t$-layer of encoder--decoder blocks is a spline of degree $3^{t+s} + 3^t - 3^s$.
\end{enumerate}
\end{theorem}
\begin{proof}
Let $f,g : \mathbb{R}^n \to \mathbb{R}$ be splines of degree $k$. Since $x + y$, $\max (x, y)$ are linear spline, it follows from Lemma~\ref{lem:comp}\ref{it:comp} that $f+g$ and $\max(f, g)$ are splines of degree $k$. In the attention module, $K(X), Q(X), V(X)$ are linear splines, it follows from Lemma~\ref{lem:comp}\ref{it:mult} that $K(X)^\tp Q(X)$ is a quadratic spline. Hence $\ReLU(K(X)^\tp Q(X)) = \max(K(X)^\tp Q(X),0) $ is also a quadratic splines and $\alpha(X) = V(X)\ReLU(K(X)^\tp Q(X))$ is a cubic spline. Similarly, the masked attention $\beta(X)$ and encoder--decoder attention $\gamma(X,Y)$ in \eqref{eq:att1} are also cubic splines. Note that the encoder--decoder attention is a quadratic spline with respect to the first variable $X$, and a linear spline with respect to the second variable $Y$; but overall it is a cubic spline with respect to $(X,Y)$.

A  neural network is a linear spline by Theorem~\ref{thm:relulinear}. So the encoder block in \eqref{eq:encblk} and decoder block in \eqref{eq:decblk} remain cubic splines by Lemma~\ref{lem:comp}\ref{it:comp}. The encoder--decoder block $\tau(X, Y) = \varphi(\gamma(X, \beta(Y)))$ is quadratic in $X$ and cubic in $Y$, and thus quintic in $(X,Y)$. Since a $t$-layer encoder or decoder is a composition of (masked) attention modules and neural networks, it is a spline of degree $3^t$. For an encoder--decoder with $s$ layers of encoder blocks and $t$ layers of encoder--decoder blocks, induction on $t$ gives $2\times 3^s + 3\times (3^{t+s-1} + 3^{t-1} - 3^s) = 3^{t+s} + 3^t - 3^s$
as its degree.
\end{proof}
The splines in \ref{it:mask}, \ref{it:edatt}, \ref{it:decblk}, \ref{it:dec} are autoregressive and those in \ref{it:edblk} and \ref{it:encdec}  partially autoregressive. The term ``autoregressive spline'' does appear in the literature but it is used in a sense entirely unrelated to \eqref{eq:auto}. We will have more to say about this in Corollary~\ref{cor:auto}.

\subsection{Veronese map}\label{sec:vero}

The degree-$k$ Veronese embedding $v_k$ is a well-known map in algebraic geometry \cite[pp.~23--25]{Harris} and polynomial optimization \cite[pp.~16--17]{Lasserre}.  Informally it is the map that takes variables $x_1,\dots,x_n$ to the monomials of degree not more than $k$ in $x_1,\dots,x_n$. This defines an injective smooth function
\begin{equation}\label{eq:ver}
v_k: \mathbb{R}^n \to \mathbb{R}^{\binom{n+k}{k}},\quad
(x_1, \dots, x_n) \mapsto (1, x_1, \dots, x_n, x_1^2, x_1x_2, \dots, x_n^k).
\end{equation}
The value $\binom{n+k}{k}$ gives the number of monomials  in $n$ variables of degree not more than $k$. Two simple examples: $v_k :  \mathbb{R} \to \mathbb{R}^k$,
$v_k(x) = (1, x, x^2, \dots, x^k)$;
$v_2 :  \mathbb{R}^2 \to \mathbb{R}^6$,
$v_2(x,y) = (1, x, y, x^2, xy, y^2)$.

In algebraic geometry \cite[pp.~23--25]{Harris} the Veronese map is usually defined over projective spaces whereas in polynomial optimization  \cite[pp.~16--17]{Lasserre} it is usually defined over affine spaces as in \eqref{eq:ver}. Nevertheless this is a trivial difference as the former is just a homogenized version of the latter.

As is standard in algebraic geometry and polynomial optimization alike, we leave out the domain dependence from the notation $v_k$ to avoid clutter, e.g., the quadratic Veronese map $v_2 :  \mathbb{R}^2 \to \mathbb{R}^6$ and $v_2 :  \mathbb{R}^6 \to \mathbb{R}^{28}$ are both denoted by $v_2$. This flexibility allows us to compose Veronese maps and speak of $v_k \circ v_{k'}$ for any $k,k' \in \mathbb{N}$. For example we may write $v_2 \circ v_2 : \mathbb{R}^2 \to \mathbb{R}^{28}$, using the same notation $v_2$ for two different maps.

The Veronese map is also defined over matrix spaces: When applied to matrices, the Veronese map simply treats the coordinates of an $n \times p$ matrix as $np$ variables. So $v_k: \mathbb{R}^{n \times p} \to \mathbb{R}^{\binom{np+k}{k}}$ is given by
\[
v_k(X ) =  (1, x_{11}, x_{12}, \dots, x_{np}, x_{11}^2, x_{11}x_{12}, \dots, x_{np}^k).
\]
For example  $v_2 :  \mathbb{R}^{2 \times 2} \to \mathbb{R}^{15}$ evaluated on $\begin{bsmallmatrix}x & y \\ z & w \end{bsmallmatrix}$ gives
\[
(1, x, y, z, w, x^2, xy, xz, xw, y^2, yz, yw, z^2, zw, w^2).
\]

An important observation for us is the following.
\begin{lemma}\label{lem:comp1}
Let $k, k' \in \mathbb{N}$. Then every coordinate of $v_{kk'}(X)$ occurs in $v_k(v_{k'}(X))$.
\end{lemma}
\begin{proof}
This is a consequence of the observation that any monomial of degree not more than $kk'$ can be written as a product of $k$ monomials, each with degree not more than $k'$.
\end{proof}

Another result that we will need is the following equivalent formulation of Pierce--Birkhoff conjecture in terms of Veronese map.
\begin{lemma}\label{lem:PB}
The Pierce--Birkhoff conjecture holds if and only if for any spline $f : \mathbb{R}^n \to \mathbb{R}$, there exist $k \in \mathbb{N}$ and a linear spline $\ell :  \mathbb{R}^{\binom{n+k}{k}} \to \mathbb{R}$ such that $f = \ell \circ v_k$.
\end{lemma}
\begin{proof}
Firstly note that Pierce--Birkhoff conjecture holds for $k =1$: Any linear spline $\ell$ can be represented in the form $\min_i \max_j \xi_{ij}$ where $\xi_{ij}$ are linear polynomials \cite{Ovchinnikov2002}. Conversely, if $\ell$ can be represented in the form $\min_i \max_j \xi_{ij}$, then it is clearly a linear spline. 

Assuming that Pierce--Birkhoff conjecture holds in general, then any polynomial spline $f$ can be written as $\min_i \max_j \xi_{ij}$, which is a linear spline over monomials of $\xi_{ij}$, i.e., $f = \ell \circ v_k$ for some linear spline $\ell$. Conversely, if every polynomial spline $f$ can be written as $\ell \circ v_k$ for some linear spline $\ell$, then since $\ell$ can always be written as $\ell = \min_i \max_j \xi_{ij}$, we have $f = \min_i \max_j \xi_{ij}\circ v_k$ for some linear polynomials $\xi_{ij}$'s. Thus we recover the statement of Pierce--Birkhoff conjecture.
\end{proof}

Observe that Lemma~\ref{lem:PB} applies verbatim to matrix-variate splines $f: \mathbb{R}^{n \times p} \to \mathbb{R}$, except that $n$ would have to be replaced by $np$ throughout and we have
\begin{equation}\label{eq:vl}
v_k: \mathbb{R}^{n \times p} \to \mathbb{R}^{\binom{np+k}{k}}, \quad \ell :  \mathbb{R}^{\binom{np+k}{k}} \to \mathbb{R}.
\end{equation}

\subsection{Splines are transformers}\label{sec:s=t}

We will show that any matrix-valued spline $f : \mathbb{R}^{n \times p} \to \mathbb{R}^{r \times p}$ is an encoder. First we will prove two technical results. We will use $i,\ih,\ib, j, \jh, \jb$ to distinguish between indices. We remind the reader that $x^+ \coloneqq \ReLU(x)$.
\begin{lemma}[Quadratic Veronese as encoders]\label{lem:main2}
Let $v_2 : \mathbb{R}^{n \times p}  \to \mathbb{R}^{(np+2)(np+1)/2}$ be the quadratic Veronese map.
There exists a two-layer encoder $\varepsilon_2 : \mathbb{R}^{n \times p} \to \mathbb{R}^{n_2 \times p}$ such that every column of $\varepsilon_2(X)$ contains a copy of $v_2(X)$ in the form
\[
\varepsilon_2(X) =
\begin{bmatrix}
v_2(X) & 0 & \cdots & 0 \\
0 & v_2(X) & \cdots & 0 \\
\vdots & \vdots & \ddots & \vdots \\
0 & 0 & \cdots & v_2(X) \\
\end{bmatrix}  \in \mathbb{R}^{n_2 \times p}.
\] 
More precisely, there is a $h_1$-headed attention module
$\alpha_1 :  \mathbb{R}^{n \times p} \to \mathbb{R}^{mh_1 \times p}$,
a one-layer  neural network,
$\varphi_1 : \mathbb{R}^{mh_1 \times p} \to \mathbb{R}^{n_1 \times p}$,
a $h_2$-headed attention module
$\alpha_2 :  \mathbb{R}^{n_1 \times p} \to \mathbb{R}^{mh_2 \times p}$,
and another one-layer  neural network
$\varphi_2 : \mathbb{R}^{mh_2 \times p} \to \mathbb{R}^{n_2 \times p}$,
such that
\begin{equation}\label{eq:enc2}
\varepsilon_2 = \varphi_2 \circ \alpha_2 \circ \varphi_1 \circ \alpha_1.
\end{equation}
In particular, any monomial of degree not more than two in the entries of $X$ appears in every column of $\varepsilon_2(X)$.
\end{lemma}
\begin{proof}
We will first construct  a multihead attention module $\alpha$ with the property that each of the $p$ columns of $\alpha(X)$ contains every entry of $X$, i.e., $x_{ij}$, $i=1,\dots,n$, $j=1,\dots,p$. Fix any $(\ih,\jh, j)$ and consider the single-head attention module $\alpha_{\ih\jh j} : \mathbb{R}^{n \times p} \to \mathbb{R}^{m \times p}$ as in \eqref{eq:att} with
\[
A_V = E_{1\ih}, \quad B_V = 0, \quad A_K = 0,\quad
B_K = E_{1\jh}, \quad A_Q = 0, \quad B_Q = E_{1j},
\]
as in \eqref{eq:QKV}. Then the $(1,j)$th entry of $\alpha_{\ih\jh j}(X)$ is exactly $x_{\ih\jh}$ and all other entries in the first row are zeros. If we repeat this for all $\ih \in \{1,\dots,n\}$, $\jh, j \in \{1,\dots,p\}$ and stack these $np^2$ attention modules together, we obtain the multihead attention $\alpha$. By construction any column of $\alpha(X)$ contains every entry of $X$.

For the required $\alpha_1$, we need to augment $\alpha$ so that every column of $\alpha_1(X)$ will also contain the constant $1$. 
Consider the single-head attention module $\alpha^{(j)} : \mathbb{R}^{n \times p} \to \mathbb{R}^{m \times p}$ with
\[
A_V = 0, \quad B_V = E_{11}, \quad A_K = 0,\quad
B_K = E_{11}, \quad A_Q = 0, \quad B_Q = E_{1j}.
\]
Then the $(1,j)$th entry of $\alpha^{(j)}(X)$ is $1$, and all other entries in the first row are zeros. We repeat this for all $j \in \{1,\dots,p\}$ and stack $\alpha^{(1)},\dots,\alpha^{(p)}$ with $\alpha$ to obtain the required $\alpha_1$. Note that $\alpha_1$ has $h_1 = np^2 + p$ heads.

Because $x = \ReLU(x) - \ReLU(-x)$, the coordinate function $f(x_1, \dots, x_n) = x_i = \ReLU(x_i) - \ReLU(-x_i)$ can be represented using a one-layer  neural network. So there exists a one-layer neural network $\varphi_1$ that only keeps all first rows of the above attention modules, and $\varepsilon_1 = \varphi_1 \circ \alpha_1$ gives
\[
\varepsilon_1(X) =
\begin{bmatrix}
v_1(X) & 0 & \cdots & 0 \\
0 & v_1(X) & \cdots & 0 \\
\vdots & \vdots & \ddots & \vdots \\
0 & 0 & \cdots & v_1(X)
\end{bmatrix}  \in \mathbb{R}^{n_1 \times p}.
\]  

We will now construct $\alpha_2$. We first repeat the construction above so that the first $np^2+p$ heads of $\alpha_2$ will produce all linear monomials (i.e., the entries of $X$), and the constant. In particular, by the end of our construction, every column of $\alpha_2 \circ \varphi_1 \circ \alpha_1(X)$ will contain every entry of $X$ and $1$, and each of these entries is the only nonzero entry in its row.

We then construct the next batch of heads of $\alpha_2$ that will produce all quadratic monomials. Consider the attention module $\alpha_{\ih\ib j}$ defined by
\[
A_V = E_{1\ih}, \quad B_V = 0, \quad A_K =0, \quad
B_K = E_{1j}, \quad A_Q =E_{1\ib}, \quad B_Q = 0.
\]
Then the $(1,j)$th entry of $\alpha_{\ih\ib j}(X)$ is $x_{\ih j}(x_{\ib j})^+$. In other words we can form quadratic terms in $j$th column out of entries in $j$th column. If we repeat this for all $\ih, \ib \in \{1,\dots,mh_1\}$, $ j \in \{1,\dots,p\}$, and stack these attention modules together, we obtain a multihead attention $\alpha_2^+$.

By our previous construction, among the rows of $\varphi_1 \circ \alpha_1(X)$ are two with only the $j$th column nonzero, taking values $x_{\ih\jh}$ and $x_{\ib\jb}$ respectively. Composing with $\alpha_{2}$, we obtain a row with $j$th entry $x_{\ih \jh}(x_{\ib \jb})^+$, and other entries zeros. So the composition $\alpha_2^+ \circ \varphi_1 \circ \alpha_1(X)$ contains all quadratic terms of the form $x_{\ih\jh}(x_{\ib\jb})^+$ in any column, and each of those entries is the only nonzero entry in its row. We may repeat the same argument to obtain  a multihead attention $\alpha_2^-$ with the property that  the composition $\alpha_2^- \circ \varphi_1 \circ \alpha_1(X)$ contains all quadratic terms of the form  $x_{\ih\jh}(-x_{\ib\jb})^+$ in any column, and each of those entries is the only nonzero entry in its row. The required $\alpha_2$ is then obtained by stacking $\alpha_2^+$, $\alpha_2^-$, together with the first $np^2 + p$ heads that give the linear monomials and constant. 

The one-layer  neural network $\varphi_2$ is then chosen so that it gives the quadratic monomial
\[
x_{\ih\jh}x_{\ib\jb} = x_{\ih\jh}(x_{\ib\jb})^+ - x_{\ih\jh}(-x_{\ib\jb})^+,
\]
for $\ih,\ib \in \{1,\dots,n\}$ and $\jh, \jb \in \{1,\dots,p\}$.
\end{proof}

Recall from Section~\ref{sec:nn} that whenever a neural network takes a matrix input $X = [x_1,\dots,x_p] \in \mathbb{R}^{n \times p}$, it is applied columnwise to each column $x_j \in \mathbb{R}^n$. In general an attention module and a neural network are distinct objects. But there is one special case when they are related.
\begin{lemma}[One-layer  neural networks as encoder blocks]\label{lem:1layer}
Let $\varphi: \mathbb{R}^n \to  \mathbb{R}^{n_2}$ be a one-layer neural network. Then
\[
\varphi : \mathbb{R}^{n \times p} \to \mathbb{R}^{n_2 \times p},\quad
[x_1,\dots,x_p] \mapsto [\varphi(x_1),\dots,\varphi(x_p)],
\]
is an encoder block of the form
$\varphi_1 \circ \alpha_1$
where $\varphi_1$ is also a one-layer neural network and $\alpha_1$ is an attention module.
\end{lemma}
\begin{proof}
Consider the attention module $\alpha_{ij} : \mathbb{R}^{n \times p} \to \mathbb{R}^{m \times p}$ given by
\[
A_V = E_{1i}, \quad B_V = 0, \quad A_K = 0, \quad
B_K = E_{1j}, \quad A_Q = 0, \quad B_Q = E_{1j}.
\]
The first row of $\alpha_{ij}(X)$ has $x_{ij}$ in its $(1, j)$th entry and zeros elsewhere. If we stack these attention modules $\alpha_{ij}$, $i \in \{1,\dots,n\}$, $ j\in \{1,\dots,p\}$ together, we obtain an $np$-headed attention module $\alpha : \mathbb{R}^{n \times p} \to \mathbb{R}^{mnp \times p}$. By construction, $\alpha(X)$ contains a submatrix of the form 
\begin{equation}\label{eq:submat}
\begin{bmatrix}
  x_1 & 0 & \cdots & 0 \\
  0 & x_2 & \cdots & 0 \\
  \vdots  & \vdots  & \ddots & \vdots  \\
  0 & 0 & \cdots & x_p 
\end{bmatrix} \in \mathbb{R}^{np \times p},
\end{equation}
where $x_j \in \mathbb{R}^n$ is the $j$th column of $X \in \mathbb{R}^{n \times p}$, $j =1,\dots,p$.
Let $\psi : \mathbb{R}^{np} \to \mathbb{R}^p$ be the affine map given by
\[
\mathbb{R}^{np} \ni \begin{bmatrix} x_1 \\ x_2 \\ \vdots \\ x_p \end{bmatrix} \mapsto x_1 + x_2 + \dots + x_p \in \mathbb{R}^n.
\]
We apply $\psi$ columnwise to $\alpha(X)$, extending its domain so that $\psi$ maps every row outside the submatrix in \eqref{eq:submat} to zero. Then the submatrix in \eqref{eq:submat} is transformed as
\[
\begin{bmatrix}
  x_1 & 0 & \cdots & 0 \\
  0 & x_2 & \cdots & 0 \\
  \vdots  & \vdots  & \ddots & \vdots  \\
  0 & 0 & \cdots & x_p 
\end{bmatrix} \mapsto 
[ x_1 , x_2 , \dots, x_p] = X
\]
and every other row outside of this submatrix gets mapped to zero. In other words $\psi$ is a left inverse of $\alpha$. The required statement then follows from
$(\varphi\circ \psi) \circ \alpha = \varphi\circ (\psi \circ \alpha) = \varphi$,
with $\varphi_1 = \varphi \circ \psi$ and $\alpha_1 = \alpha$.
\end{proof}
While the definition of an encoder as in Section~\ref{sec:enc} does not require the neural networks within it to have only one hidden layer, the original version in \cite{vaswani2017attention} does. Lemma~\ref{lem:1layer} shows that this is not really a more stringent requirement since whenever we are presented with a multilayer  neural network we may repeatedly apply Lemma~\ref{lem:1layer} to turn it into the form required in \cite{vaswani2017attention}.

Assuming the Pierce--Birkhoff conjecture, we may now show that any matrix-valued spline $f : \mathbb{R}^{n \times p} \to \mathbb{R}^{r \times p}$ is an encoder. We prove the most general case possible so that other special cases follow effortlessly: the corresponding result for vector-valued splines is obtained by setting $p =1$ and that for scalar-valued splines by setting $r = p = 1$. Note also that the result below applies to splines defined on any semialgebraic partition --- the most common rectilinear partition obtained through triangular of domain is also a special case.

\begin{theorem}[Splines as encoders]\label{thm:main2}
Let $f : \mathbb{R}^{n \times p} \to \mathbb{R}^{r \times p}$ be a max-definable function. Then $f$ is a $t$-layer encoder for some finite $t \in \mathbb{N}$. More precisely, there exist $t$ attention modules $\alpha_1,\dots,\alpha_t$ and $t$ one-layer neural networks $\varphi_1,\dots,\varphi_t$
such  that
\[
f = \varphi_t \circ \alpha_t \circ \varphi_{t-1} \circ \alpha_{t-1} \circ \dots \circ\varphi_1 \circ \alpha_1.
\]
If the Pierce--Birkhoff conjecture holds, then any degree-$k$ spline is an encoder.
\end{theorem}
\begin{proof}
Let $f(X) = [f_1(X), \dots, f_p(X)] \in \mathbb{R}^{r \times p}$ with $f_j(X) \in \mathbb{R}^r$, $j =1,\dots, p$. By Lemma~\ref{lem:PB}, we may write $f_j = \ell_j \circ v_{k_j}$ where $\ell_j$ is a linear spline and $v_{k_j}$ the Veronese map of degree $k_j$. Let $s \coloneqq \max ( k_1,\dots, k_p)$. Then by padding  $\ell_j$ with extra terms with zero coefficients, we may assume
\begin{equation}\label{eq:flv}
f_j = \ell_j \circ v_{k_j} = \ell_j \circ v_s.
\end{equation}
Note that $\ell_j :  \mathbb{R}^{\binom{np + s}{s}} \to \mathbb{R}^r$.

It follows from Lemma~\ref{lem:comp1} that we may obtain all monomials of degree not more than $s$ by composing the quadratic Veronese map with itself sufficiently many times. So by composing $\lceil \log_2 s \rceil$ copies of the encoder constructed in Lemma~\ref{lem:main2}, we obtain an encoder
\begin{equation}\label{eq:enck}
\varepsilon \coloneqq \varepsilon_2 \circ \dots \circ \varepsilon_2
\end{equation}
with the property that any column of $\varepsilon(X)$ contains a copy of Veronese map of degree $s$, i.e.,
\[
\varepsilon(X) =
\begin{bmatrix}
v_s(X) & 0 & \cdots & 0 \\
0 & v_s(X) & \cdots & 0 \\
\vdots & \vdots & \ddots & \vdots \\
0 & 0 & \cdots & v_s(X)
\end{bmatrix}  \in \mathbb{R}^{n_t \times p}.
\]
There is a slight abuse of notation in \eqref{eq:enck}: We have assumed that the $(i+1)$th copy of $\varepsilon_2$ has input dimension $n_{i+1} \coloneqq (n_i p+2)(n_i p+1)/2$, the output dimension of the $i$th copy of $\varepsilon_2$. Strictly speaking these are different maps since domains and codomains are different although we denote all of them as $\varepsilon_2$. Also, in the final layer, we drop any rows that we do not need --- this is not a problem as ``dropping rows'' is just a modification of the  neural network in the last layer $\varphi_t$, which we will be modifying anyway below.

Expanding each copy of $\varepsilon_2$ as in \eqref{eq:enc2}, we obtain the structure in \eqref{eq:enc}, i.e.,
\begin{equation}\label{eq:enck2}
\varepsilon = \varphi_t \circ \alpha_t \circ \varphi_{t-1} \circ \alpha_{t-1} \circ \dots \circ\varphi_1 \circ \alpha_1
\end{equation}
for some $t \in \mathbb{N}$. We will modify the attention module $\alpha_t : \mathbb{R}^{n_t \times p} \to \mathbb{R}^{m_t \times p}$ in the last layer. For $i =1,\dots,2r$, we let $\alpha^{(i)} : \mathbb{R}^{n_t  \times p} \to \mathbb{R}^{m \times p}$ be a (single-head) attention module with
\[
A_V^{(i)} =0,\quad A_K^{(i)} = 0,\quad A_Q^{(i)} = 0,\quad B_V^{(i)} = E_{11},\quad  B_K^{(i)} = E_{11},
\]
and $B_Q^{(i)}$ is a nonnegative constant matrix to be determined later. The first row of $\alpha^{(i)}(X)$ is the first row of $B_Q^{(i)}$, i.e., $\alpha^{(i)}(X)$ contains a row of nonnegative constants. By stacking $2r$ heads $\alpha^{(1)}, \dots, \alpha^{(2r)}$ onto $\alpha_t$, we obtain a modified attention module with $2r$ extra heads,
\[
\widehat{\alpha}_t \coloneqq (\alpha^{(1)},\dots,\alpha^{(2r)}, \alpha_t) : \mathbb{R}^{n_t \times p} \to \mathbb{R}^{(2mr + m_t)\times p}
\]
Prefixing these heads to $\alpha_t$ will allow us to add $2r$ rows of nonnegative constants to $\varepsilon(X)$.

First, by modifying the  neural network $\varphi_t$ to $\widehat{\varphi}_t$, one that keeps the first row of each those $2r$ extra heads, we see that $\widehat{\varphi}_t \circ \widehat{\alpha}_t(X)$ will have $2r$ rows of nonnegative constants irrespective of $X$. We may also choose $\widehat{\varphi}_t$ so that these occur as the first through $2r$th rows, denoted as
\[
\begin{bmatrix}
b_1 & b_2 & \cdots & b_p \\
b_1' & b_2' & \cdots & b_p'
\end{bmatrix} \in \mathbb{R}^{2r \times p},
\]
for some $b_i, b_i' \in \mathbb{R}_+^r$, $i=1,\dots,p$. Note that each row of the matrix above comes from one of the added heads $\alpha^{(1)}, \dots, \alpha^{(2r)}$.

By replacing $\varphi_t$ and $\alpha_t$ in \eqref{eq:enck2} with $\widehat{\varphi}_t$ and $\widehat{\alpha}_t$, we obtain an encoder $\widehat{\varepsilon}  : \mathbb{R}^{n \times p} \to \mathbb{R}^{(2r+n_t) \times p}$,
\[
\widehat{\varepsilon} \coloneqq \widehat{\varphi}_t \circ \widehat{\alpha}_t \circ \varphi_{t-1} \circ \alpha_{t-1} \circ \dots \circ\varphi_1 \circ \alpha_1.
\]
By our construction we must have
\[
\widehat{\varepsilon}(X) =
\begin{bmatrix}
b_1 & b_2 & \cdots & b_p \\
b_1' & b_2' & \cdots & b_p' \\
v_s(X) & 0 & \cdots & 0 \\
0 & v_s(X) & \cdots & 0 \\
\vdots & \vdots & \ddots & \vdots \\
0 & 0 & \cdots & v_s(X)
\end{bmatrix} \in \mathbb{R}^{(2r+n_t) \times p}.
\]

Define the linear spline $\ell :  \mathbb{R}^{2r + \binom{np + s}{s}p} \to \mathbb{R}^{(p+2)r}$ by
\[
\ell(x, x', x_1, \dots, x_p) = \bigl(x, x', \ell_1(x_1), \dots, \ell_p(x_p)\bigr).
\]
Here $x, x' \in \mathbb{R}^r$, $x_1, \dots, x_p \in \mathbb{R}^{\binom{np + s}{s}}$.
By Theorem~\ref{thm:relulinear}, linear splines are exactly  neural networks. Recall from Section~\ref{sec:nn} that when we apply a  neural network  to a matrix, we apply it columnwise. Hence
\[
\ell \circ \widehat{\varepsilon}(X) =
\begin{bmatrix}
b_1 & b_2 & \cdots & b_p \\
b_1' & b_2' & \cdots & b_p' \\
f_1(X) & \ell_1(0) & \cdots & \ell_1(0) \\
\ell_2(0) & f_2(X) & \cdots & \ell_2(0) \\
\vdots & \vdots & \ddots & \vdots \\
\ell_p(0) & \ell_p(0) & \cdots & f_p(X)
\end{bmatrix},
\]
where we have used \eqref{eq:flv}. Now we set
\[
b_i = \ReLU\Bigl(-\sum_{j \neq i} \ell_j(0)\Bigr), \quad
b_i' = \ReLU\Bigl(\sum_{j \neq i} \ell_j(0)\Bigr),
\]
for each $i =1,\dots,p$.
Let $\psi: \mathbb{R}^{(p+2)r} \to \mathbb{R}^r$ be the linear map defined by
\[
\psi(y, y', y_1, \dots, y_p) = y - y' + y_1+ \dots +y_p,
\]
where $y, y', y_1, \dots, y_p \in \mathbb{R}^r$. 
Then the composition of $\psi \circ \ell \circ \widehat{\varepsilon}$ has
\[
\psi \circ \ell \circ \widehat{\varepsilon}(X) = [f_1(X), \dots, f_p(X)] = f(X),
\] 
as required. At this point we have obtained $f$ as an encoder according to the definition in Section~\ref{sec:enc} since $\psi \circ \ell \circ \widehat{\varphi}_t$ is clearly a multilayer neural network. By our remark after the proof of Lemma~\ref{lem:1layer}, it may be converted into an alternate composition of attention modules and single-layer neural networks.
\end{proof}
In case the reader is wondering the value of $s$ in the proof above is not necessarily $k$ and can be strictly larger. To the best of our knowledge, there is not even a \emph{conjectural}  effective version of Conjecture~\ref{conj:PB} in the literature. So unlike Theorem~\ref{thm:relulinear}, any bounds on the number of encoder blocks, number of heads of attention modules, width of the neural networks, etc, are beyond reach  at this point.

Just as Theorem~\ref{thm:relulinear} establishes the equivalence between $\ReLU$-neural networks and linear splines, various parts of the results in this article collectively establish the equivalence between $\ReLU$-encoders and splines, assuming the validity of the Pierce--Birkhoff conjecture.
\begin{corollary}
If the Pierce--Birkhoff conjecture holds, then the following classes of functions are all equal:
\begin{enumerate}[label=\upshape{(\roman*)}, nosep]
\item splines;

\item encoders;

\item max-definable functions;

\item linear splines composed with the Veronese map.
\end{enumerate}
\end{corollary}

While our article is about understanding transformers in terms of splines, there is a somewhat unexpected payoff: the proof of Theorem~\ref{thm:main2} yields a way to construct autoregressive splines. There appears to be no universally agreed-upon meaning for the term ``autoregressive spline'' in the existing literature. In particular none replicates \eqref{eq:auto} and we are unaware of any construction that yields a spline that is autoregressive in the sense of \eqref{eq:auto}.

\begin{corollary}[Autoregressive splines as decoders]\label{cor:auto}
Let $k \in \mathbb{N}$ and $f : \mathbb{R}^{n \times p} \to \mathbb{R}^{m \times p}$ be an autoregressive max-definable function. Then $f$ is a $t$-layer decoder for some finite $t \in \mathbb{N}$. More precisely, there exist $t$ masked attention modules $\beta_1,\dots,\beta_t$ and $t$ one-layer neural networks $\varphi_1,\dots,\varphi_t$ such  that
\[
f = \varphi_t \circ \beta_t \circ \varphi_{t-1} \circ \beta_{t-1} \circ \dots \circ\varphi_1 \circ \beta_1.
\]
If the Pierce--Birkhoff conjecture holds, then any degree-$k$ autoregressive spline is a decoder.
\end{corollary}
\begin{proof}
The proof of Lemma~\ref{lem:main2} applies almost verbatim. In fact it is slightly simpler since in the $(1,j)$th entry, we only need to construct monomials of the form $x_{\ih\jh}$, $x_{\ih\jh} x_{\ib\jb}$ for $\jh, \jb \le j$. The same constructions used to obtain $A_V, B_V, A_K, B_K, A_Q, B_Q$ produce these required monomials when we use masked attention modules in place of attention modules. The proofs of Lemma~\ref{lem:1layer} and Theorem~\ref{thm:main2} then apply with masked attention modules in place of attention modules.
\end{proof}
A similar construction can be extended to construct partially autoregressive splines as encoder--decoders.

\section{Conclusion}\label{sec:conc}

It is an old refrain in mathematics that one does not really understand a mathematical proof until one can see how every step is inevitable. This is the level of understanding that we hope   Section~\ref{sec:s=t} provides for the transformer.

\subsection{Insights}

Arora et al.\ \cite{arora2018understanding} have shown that neural networks are exactly linear splines. Since compositions of linear splines are again linear splines, to obtain more complex functions we need something in addition to neural networks. Viewed in this manner, the attention module in Section~\ref{sec:att} is the simplest function that serves the role. Lemma~\ref{lem:main2} shows that the quadratic Veronese map, arguably the simplest map that is not a linear spline, can be obtained by composing two attention modules. The proof reveals how heads and layers are essential: It would fail if we lacked the flexibility of having multiple heads and layers. The proof also shows how a neural network works hand-in-glove with attention module: It would again fail if we lack either one. The proof of Theorem~\ref{thm:main2} then builds on Lemma~\ref{lem:main2}: By composing quadratic Veronese maps we can obtain Veronese map of any higher degree; and by further composing it with linear splines we obtain all possible splines. The resulting map, an alternating composition of attention modules and neural networks, is exactly the encoder of a transformer.

There are some other insights worth highlighting. Lemma~\ref{lem:1layer} explains why the neural networks within a transformer require no more than one hidden layer; Vaswani et al.\ \cite{vaswani2017attention} likely arrived at this same conclusion through their experimentation. Theorem~\ref{thm:main1}\ref{it:enc} shows why layering attention modules and neural networks makes for an effective way to increase model complexity --- the degree of the spline $3^t$ increases exponentially with the number of layers $t$.

\subsection{Recommendations}

Recent work of Wortsman et al.\ \cite{wortsman2023replacing} shows that a $\ReLU$-transformer is perfectly capable of achieving results of similar quality as the original $\SoftMax$-transformer, offering significant computational savings. We also advocate the use of $\ReLU$ activation, if only for turning a nearly-mystical and sometimes-feared technology into a familiar friendly one. In which case we could drop the word ``smoothed'' in our title --- attention is a cubic spline. 

If a smooth function is desired, we argue for using $\SoftPlus$ instead of $\SoftMax$ as activation. The $\SoftMax$ function is the natural smooth proxy for $\argmax$ as well as the derivative of $\SoftPlus$, also known as the log-sum-exp function, which is in turn the natural smooth proxy for $\ReLU$. Indeed $\SoftPlus$ has been used in place of $\ReLU$ to construct smooth neural networks with encouraging results \cite{Gaubert}. Despite their intimate relationship, $\SoftMax$ makes for a poor proxy for $\ReLU$. On the basis of our work, a $\SoftPlus$-activation would be natural, smooth, and preserves fidelity with splines.

Lastly, Section~\ref{sec:s=t} points to the importance of a nearly forgotten seventy-year-old conjecture about splines by one of its pioneers.  Indeed, Theorem~\ref{thm:main2} shows that the Pierce--Birkhoff conjecture is true if and only if every spline is an encoder. Perhaps this article will rekindle interest in the conjecture and point a way towards its resolution.



\bibliographystyle{abbrv}
\bibliography{ref}

\end{document}